%% file: main.tex
\def\isarxivversion{1} 

\ifdefined\isarxivversion
\documentclass[11pt]{article}
\else
\fi

\ifdefined\isneurips
\documentclass{article}
\usepackage{neurips_2020}
\else
\fi

\ifdefined\isitcs
\documentclass[a4paper,UKenglish,cleveref, autoref, thm-restate]{lipics-v2019}

\title{Training (Overparametrized) 
	Neural Networks in Near-Linear Time } 

\titlerunning{Training (Overparametrized) Neural Networks } 

\author{Jan van den Brand}{KTH Royal Institute of Technology, Sweden}{janvdb@kth.se}{}{This project has received funding from the European Research Council (ERC) under the European Unions Horizon 2020 research and innovation program under grant agreement No 715672. Partially supported by the Google PhD Fellowship Program.}

\author{Binghui Peng}{Columbia University, United States}{bp2601@columbia.edu}{}{Research supported by NSF IIS-1838154, NSF CCF-1703925 and NSF CCF-1763970.}

\author{Zhao Song}{Princeton University and Institute for Advanced Study, United States}{zhaos@ias.edu}{}{Research supported by Special Year on Optimization, Statistics, and Theoretical Machine Learning (being led by Sanjeev Arora) at Institute for Advanced Study.}

\author{Omri Weinstein}{Columbia University, United States}{omri@cs.columbia.edu}{}{Research supported by NSF CAREER award CCF-1844887.}

\authorrunning{J.\,v.d.Brand, B.\, Peng, Z.\, Song and O.\, Weinstein} 

\Copyright{J.\,v.d.Brand, B.\, Peng, Z.\, Song and O.\, Weinstein} 



\category{} 

\relatedversion{} 

\supplement{}

\acknowledgements{The author would like to thank David Woodruff for telling us the tensor trick for computing kernel matrices and helping us improve the presentation of the paper.}

\EventEditors{James R. Lee}
\EventNoEds{1}
\EventLongTitle{12th Innovations in Theoretical Computer Science Conference (ITCS 2021)}
\EventShortTitle{ITCS 2021}
\EventAcronym{ITCS}
\EventYear{2021}
\EventDate{January 6--8, 2021}
\EventLocation{Virtual Conference}
\EventLogo{}
\SeriesVolume{185}
\ArticleNo{68}

\else
\fi

\usepackage{amsmath}
\usepackage{amsthm}
\usepackage{amssymb}
\usepackage{algorithm}
\usepackage{color}
\usepackage{graphicx}
\usepackage{grffile}
\usepackage{epstopdf}
\usepackage{url}
\usepackage{color}
\usepackage{epstopdf}
\usepackage{algpseudocode}
\usepackage[T1]{fontenc}
\usepackage{bbm}
\usepackage{comment}
\usepackage{dsfont}

\usepackage{tikz}
\usepackage{hyperref}  
\hypersetup{colorlinks=true,citecolor=blue,linkcolor=blue}
\usetikzlibrary{arrows}
\usepackage{xspace}

\graphicspath{{./figs/}}
\usepackage{mathtools}

\ifdefined\isarxivversion
\usepackage[margin=1in]{geometry}
\else 

\fi

\ifdefined\isitcs
\newtheorem{assumption}[theorem]{Assumption}
\newtheorem{fact}[theorem]{Fact}
\else
\newtheorem{theorem}{Theorem}[section]
\newtheorem{lemma}[theorem]{Lemma}
\newtheorem{definition}[theorem]{Definition}

\newtheorem{assumption}[theorem]{Assumption}

\newtheorem{fact}[theorem]{Fact}
\newtheorem{remark}[theorem]{Remark}

\fi

\renewcommand{\lg}{\log}

\newcommand{\wh}{\widehat}
\newcommand{\wt}{\widetilde}

\newcommand{\eps}{\epsilon}

\newcommand{\cT}{\mathcal{T}}
\newcommand{\R}{\mathbb{R}}

\renewcommand{\varepsilon}{\epsilon}
\renewcommand{\tilde}{\wt}
\renewcommand{\hat}{\wh}

\renewcommand{\eps}{\epsilon}
\renewcommand{\d}{\mathsf{d}}
\newcommand{\poly}{\mathrm{poly}}

\DeclareMathOperator*{\E}{{\mathbb{E}}}

\ifdefined\debug
 \newcommand{\Zhao}[1]{{\color{red}[Zhao: #1]}}
 \newcommand{\Binghui}[1]{{\color{blue}[Binghui: #1]}}
 \newcommand{\Omri}[1]{{\color{purple}[Omri: #1]}}
 \newcommand{\Jan}[1]{{\color{orange}[Jan: #1]}}
\else
 \newcommand{\Zhao}[1]{}
 \newcommand{\Binghui}[1]{}
 \newcommand{\Omri}[1]{}
 \newcommand{\Jan}[1]{}
\fi

\begin{document}

\ifdefined\isitcs
\else
\title{
Training (Overparametrized) 
Neural Networks \\in Near-Linear Time \thanks{A preliminary version of this paper appeared in the Proceedings of the 12th Innovations in Theoretical Computer Science (ITCS 2021).}
}
\author{
	Jan van den Brand\thanks{\texttt{janvdb@kth.se}. KTH Royal Institute of Technology. This project has received funding from the European Research Council (ERC) under the European Unions Horizon 2020 research and innovation programme under grant agreement No 715672.}
	\and
	Binghui Peng\thanks{\texttt{bp2601@columbia.edu}. Columbia University. Research supported by NSF IIS-1838154, NSF CCF-1703925 and NSF CCF-1763970}
	\and 
	Zhao Song\thanks{\texttt{zhaos@ias.edu}. Princeton University and Institute for Advanced Study. Part of the work done while visiting Columbia University and hosted by Omri Weinstein. Research supported by Special Year on Optimization, Statistics, and Theoretical Machine Learning (being led by Sanjeev Arora) at Institute for Advanced Study.}
	\and
	Omri Weinstein\thanks{\texttt{omri@cs.columbia.edu}. Columbia University. Research supported by NSF CAREER award CCF-1844887.}
}
\date{}
\fi

\ifdefined\isarxivversion

\begin{titlepage}
\maketitle
\begin{abstract}

\input{abstract}
\end{abstract}
\thispagestyle{empty}
\end{titlepage}

\else

\maketitle
\begin{abstract}
\input{abstract}
\end{abstract}

\fi

\input{intro}   
\input{tech}    
\input{preli}   
\input{result}  
\input{concl}   

\section*{Acknowledgments}
The authors would like to thank David Woodruff for telling us the tensor trick for computing kernel matrices and helping us improve the presentation of the paper.
The authors would like to thank Sanjeev Arora, Simon S. Du, and Jason Lee for the suggestion of this topic. 
The authors would like to thank Yangsibo Huang, Shunhua Jiang, Yaonan Jin, Kai Li, Xiaoxiao Li, Zhenyu Song, Yushan Su, Fan Yi, and Hengjie Zhang for very useful discussions.


\ifdefined\isitcs

\else
\newpage
\appendix
\input{app}         
\input{regression}  
\input{two}         
\input{optimization}
\fi

\ifdefined\isitcs
\bibliography{ref}
\else
\newpage
{\small
\bibliographystyle{alpha}
\bibliography{ref}
}
\fi

\end{document}

%% file: abstract.tex
The slow convergence rate and pathological curvature issues of first-order gradient methods for training deep neural networks, 
initiated an ongoing effort for developing faster \emph{second-order} optimization algorithms beyond SGD, without compromising 
the generalization error. 
Despite their remarkable convergence rate (\emph{independent} of the training batch size $n$), second-order algorithms 
incur a daunting slowdown in the \emph{cost per iteration} (inverting the Hessian matrix of the loss function), which renders them impractical. Very recently, this computational overhead was mitigated by the works of \cite{zmg19,cgh+19}, yielding an 
$O(mn^2)$-time second-order algorithm for training two-layer overparametrized 
neural networks of polynomial width $m$. 

We show how to speed up the algorithm of \cite{cgh+19}, achieving an 
$\tilde{O}(mn)$-time backpropagation algorithm for training (mildly overparametrized) 
ReLU networks, which is near-linear in the dimension ($mn$) of the full gradient (Jacobian) matrix.  
The centerpiece of our algorithm is to reformulate the Gauss-Newton iteration as an 
$\ell_2$-regression problem, and then use a Fast-JL type dimension reduction to \emph{precondition} 
the underlying Gram matrix in time independent of $M$, allowing to find a sufficiently good approximate solution 
via \emph{first-order} conjugate gradient. Our result provides a proof-of-concept that advanced machinery from 
randomized linear algebra---which led to recent breakthroughs in \emph{convex optimization} 
(ERM, LPs, Regression)---can be carried over to the realm of deep learning as well.








%% file: intro.tex

\section{Introduction}
Understanding the dynamics of gradient-based optimization 
of deep neural networks has been a central 
focal point of theoretical machine learning in recent years \cite{ly17,zsjbd17,zsd17,ll18,dzps19,als19a,als19b,all19,bjw19,om19,adh+19,sy19,d20,jt20,belm20}.  
This line of work led to a remarkable rigorous understanding 
of the generalization, robustness and convergence rate of \emph{first-order} (SGD-based) algorithms, which 
are the standard choice for training DNNs. By contrast, the \emph{computational complexity} of implementing 
gradient-based training algorithms (e.g., backpropagation)  in such non-convex landscape is less understood, and gained traction 
only recently due to the overwhelming size of training data and complexity of network design \cite{mg15,dhs11,ljh+19,cgh+19,zmg19}.  

The widespread use  
first-order methods such as (stochastic) gradient descent in training DNNs is explained, 
to a large extent, by its computational efficiency -- recalculating the gradient of the loss function 
at  each iteration is simple and cheap (linear in the dimension of the full gradient), let alone with the 
advent of minibatch random sampling \cite{hrs15,cgh+19}. 
 Nevertheless, first-order 
methods have a slow rate of convergence in non-convex settings (typically $\Omega(\poly(n)\log(1/\eps))$ for 
overparametrized networks, see e.g.,~\cite{zmg19}) 
for reducing the training error below $\eps$, and it is increasingly clear that SGD-based algorithms are becoming   
a real bottleneck for many practical purposes.  
This drawback initiated a substantial effort for developing fast training methods beyond SGD, aiming to improve its 
convergence rate without compromising the generalization error \cite{bl88, m10, mg15, dhs11, kb14, pw17, 
cgh+19, zmg19}. 
 
Second-order gradient algorithms (which employ information about the Hessian of the loss function),  
pose an intriguing computational tradeoff in this context: On one hand, they are known to converge 
extremely fast, at a rate \emph{independent} of the input size (i.e., only $O(\log 1/\eps)$ iterations \cite{zmg19}),  
and offer a qualitative advantage in overcoming pathological curvature issues that arise in first-order 
methods, by exploiting the local geometry of the loss function. 
This feature implies another practical advantage of second order methods, namely, that they do not 
require tuning the learning rate \cite{cgh+19, zmg19}.  
On the other hand, second-order methods have a prohibitive \emph{cost per iteration}, as they involve  
\emph{inverting} a dynamically-changing dense Hessian matrix.
This drawback explains 
the scarcity of second order methods in \emph{large scale non-convex} optimization, 
in contrast to its popularity in the convex setting. 

The recent works of \cite{cgh+19, zmg19} addressed the computational bottleneck of 
second-order algorithms in optimizing deep neural nets, and presented a training algorithm 
for overparametrized neural networks with smooth (resp. ReLU) activations, whose running time 
is $O(mn^2)$, where $m$ is the width of the neural network, and $n$ is the size of the training data in $\R^d$. 
The two algorithms, which achieve essentially the same running time, are based on the classic 
Gauss-Newton algorithm (resp. `Natural gradient' algorithm)
combined with the recent introduction of \emph{Neural Tangent Kernels} (NTK) 
\cite{jgh18}. 
The NTK formulation utilizes a local-linearization of the loss function for overparametrized neural networks,   
which reduces the optimization problem of DNNs to that of a \emph{kernel regression} problem:   
The main insight is that when the network is \emph{overparametrized}, i.e., sufficiently wide $m \gtrsim n^4$ (\cite{sy19}),
the neural network becomes locally convex and smooth, hence the problem is 
equivalent to a kernel regression 
problem with respect to the NTK function \cite{jgh18}, and therefore solving the latter 
via (S)GD is guaranteed to converge to a global minimum.  
The training algorithm of \cite{cgh+19} draws upon this equivalence, by designing a second-order variation of the 
Gauss-Newton algorithm (termed `Gram-Gauss-Newton'),  yielding the aforementioned runtime for \emph{smooth activation functions}. 

\ifdefined\isitcs
\vspace{+2mm}
\noindent {\bf Single vs. Multilayer Network Training \ \ } 
\else
\paragraph{Single vs. Multilayer Network Training}
\fi
Following \cite{cgh+19, zmg19}, we focus on two-layer (i.e., single hidden-layer) neural networks. 
While our algorithm extends to the multilayer case (with a slight comprise on the width dependence), 
we argue that, as far as training time, the two-layer case is not only the common case, 
but in fact the \emph{only} interesting case for constant training error:  
Indeed, in the multilayer case ($L\geq 2$), we claim that the mere cost of \emph{feed-forward} computation 
of the network's output is already $\Omega_\eps(m^2nL)$. Indeed, the 
total number of parameters of $L$-layer networks is $M = (L - 1)m^2 + md$, and as such,  
feed-forward computation requires, at the very least, computing a single product of $m\times m$ (dense) matrices $W$ with a $m \times 1$ vector for each training data, which already costs $m^2n$ time: 

\[
\hat{y}_{i} = a^{\top}\sigma_{L}\left( \underbrace{W_L}_{m \times m} \sigma_{L - 1}\left(\underbrace{W_{L - 1}}_{m \times m} \ldots \sigma_{1}(\underbrace{W_1}_{m \times d} x_i)  \right)\right)
\]
 
Therefore, sublinear-time techniques (as we present) appear futile in the case of multi-layer overparametrized networks,  
where it is possible to achieve linear time (in $M$) using essentially direct (lossless) computation (see next subsection). 
It may still be possible to use sublinear algorithms to improve the running time to $O(m^2nL + \poly(n))$, though in for overparametrized 
DNNs this seems a minor saving.

\subsection{Our Result}

Our main result is a quadratic speedup to the algorithm of \cite{cgh+19}, yielding  
an \emph{essentially optimal} training algorithm for overparametrized  
two-layer neural networks. Moreover, in contrast to \cite{cgh+19},
our algorithm applies to the more complex and realistic 
case of \emph{ReLU} activation functions.  
Our main result is shown below (For a more comprehensive comparison, see Table~\ref{tab:my_label} below 
and references therein). 

\begin{theorem}
	\label{thm:second-order-intro}
	Suppose the width of a two layer ReLU neural network satisfies 
	\begin{align*} 
	m = \Omega (\max \{ \lambda^{-4} n^4, \lambda^{-2} n^2d\log(n/\delta) \} ),
	\end{align*} 
	where 
	$\lambda >0$ denotes the minimum eigenvalue of the Gram matrix (see Eq.~\eqref{eq:kernel} below), $n$ is the number of training data, $d$ is the input dimension.
	Then with probability $1 - \delta$ over the random initialization of neural network and the randomness of the training algorithm,  
	our algorithm achieves
	\begin{align*}
	\| f_{t+1} - y \|_2 \leq  \frac{1}{2}  \| f_t - y \|_2.
	\end{align*}
	The computational cost of each iteration is $\tilde{O}(mnd + n^3)$, and the running time for reducing the 
	training loss to $\eps$ is $\tilde{O}((mnd + n^3) \log(1/\eps))$.
	Using fast matrix-multiplication, the total running time can be further reduced to 
	$\tilde{O}((mnd + n^{\omega}) \log(1/\eps))$.\footnote{Here, $\omega < 2.373$ denotes the fast matrix-multiplication  (FMM) constant for multiplying 
		two $n\times n$ matrices \cite{w12,l14}.}
\end{theorem}





\begin{table}[htbp]
    \centering
    \begin{tabular}{|l|l|l|l|l|l|} \hline
         {\bf Ref.} & {\bf Method} & {\bf\#Iters} & {\bf Cost/iter} & {\bf Width} & {\bf ReLU?}\\ \hline
         \cite{dzps19} & Gradient descent & $O(n^2\log(1/\epsilon))$ & $O(mn)$ & $\Omega(n^6)$ & Yes\\ \hline
         \cite{sy19} & Gradient descent & $O(n^2\log(1/\epsilon))$ & $O(mn)$ & $\Omega(n^4)$ & Yes\\ \hline
         \cite{wdw19} & Adaptive gradient descent & $O(n\log(1/\epsilon))$ & $O(mn)$ & $\Omega(n^6)$ &Yes\\ \hline
         \cite{cgh+19} & Gram-Gaussian-Newton (GGN) & $O(\log\log(1/\epsilon))$ & $O(mn^2)$ & $\Omega(n^4)$ & No \\ \hline
         \cite{cgh+19} & Batch-GGN & $O(n^2\log(1/\eps))$ & $O(m)$ & $\Omega(n^{18})$ & No\\ \hline
         \cite{zmg19} & Natural gradient descent & $O(\log(1/\eps))$ & $O(mn^2)$ & $\Omega(n^4)$ & Yes\\\hline
        \bf{Ours} &   & $O(\log(1/\epsilon))$ & $O(mn)$ & $\Omega(n^4)$ &Yes\\ \hline 
    \end{tabular}
    
    \caption{
    Summary of state-of-art algorithms for training two-layer neural networks. 
    $n$ denotes the training batch size (number of input data points in $\R^d$) and $\epsilon$ 
    denote the desired accuracy of the training loss. For simplicity, here we assume 
    $d=O(1)$ and omit $\poly (\log n, 1/\lambda)$ terms. The result of \cite{cgh+19} applies only to  smooth activation 
    gates and not to ReLU networks. Comparison to SGD algorithms is omitted from this table since they require 
    a must stronger assumption on the width $m$ for convergence, and have slower convergence rate than GD 	
    \cite{ll18,als19a,als19b}. } 
    \label{tab:my_label}
\end{table}

\begin{remark}
We stress that that our algorithm runs in (near) linear time 
even for networks with width $m \gtrsim n^2$ and in fact, under the common belief that $\omega=2$, this  
is true so long as  $m \gtrsim n$ (!). This means that the bottleneck for linear-time training of \emph{small-width} DNNs 
is \emph{not computational}, but rather \emph{analytic}: 
The overparametrization requirements  ($m \gtrsim n^4$) in Theorem \ref{thm:second-order-intro} 
stems from current-best analysis of the convergence 
guarantees of (S)GD-based training of ReLU networks, and any improvement on these bounds would 
directly yield linear-time training for thinner networks using our algorithm.  
\end{remark}

\ifdefined\isitcs
\vspace{+2mm}
{\noindent \bf Techniques \ \ } 
\else
\paragraph{Techniques}
\fi
The majority of ML optimization literature on overparametrized network training 
is dedicated to understanding and minimizing the \emph{number of iterations} of 
the training process~\cite{zmg19, cgh+19}  
as opposed to the \emph{cost per iteration}, which is the focus of our paper. 
Our work shows that it is possible to harness the toolbox of \emph{randomized linear algebra}--- 
which was heavily used in the past decade to reduce the cost of \emph{convex optimization} tasks--- 
in the nonconvex setting of deep learning as well. A key ingredient in our algorithm is 
\emph{linear sketching}, where the main idea is to carefully \emph{compress} a linear system 
underlying an optimization problem, 
in a way that preserves  a good enough solution to the problem yet can be solved much faster in 
lower dimension. This is the essence of the celebrated \emph{Sketch-and-Solve} (S\&S)  
paradigm \cite{cw13}. 
As we explain below, 
our main \emph{departure} from the classic S\&S framework (e.g., \cite{pw17}) is that 
we cannot afford to directly solve the underlying compressed regression problem (as this 
approach turns out to be prohibitively slow for our application). Instead, we use sketching (or sampling) to facilitate 
\emph{fast preconditioning} of linear systems (in the spirit of~\cite{st04, kosz13,rt08, w14}), which in turn enables 
to solve the compressed regression problem to very high accuracy via first-order \emph{conjugate}  
gradient descent. This approach essentially \emph{decouples} the sketching error from the final precision 
error of the Gauss-Newton step, enabling a much smaller sketch size. 
We believe this (somewhat unconventional) approach to non-convex optimization 
is the most enduring message of our work.  


\subsection{Related Work}

\ifdefined\isitcs
\vspace{+2mm}
{\noindent \bf Second-order methods in non-convex optimization \ \ } 
\else
\paragraph{Second-order methods in non-convex optimization}
\fi
Despite the prevalence of first order methods in deep learning applications, there is a vast body of ongoing 
work~\cite{brb17,blh18,mg15,gm16,gks18, cgh+19, zmg19} aiming to design more scalable second-order 
algorithms that overcome the limitations of (S)GD for optimizing deep models.
Grosse and Martens~\cite{mg15, gm16} designed the K-FAC method, where the idea is to use Kronecker-factors to 
approximate the Fisher information matrix, combined with natural gradient descent. This approach has been further 
explored and extended by~\cite{wmg+17, glb+18,mbj18}. Gupta et al.~\cite{gks18} designed the ``Shampoo method'', 
based on the idea of {\em structure-aware preconditioning}. Anil et al.~\cite{agk+20} further validate the practical perfromance of 
Shampoo and incorporated it into hardware. However, despite sporadic empirical evidence of such second-order methods 
(e.g., K-FAC and Shampoo), these methods generally lack a provable theoretical guarantee on the performance 
when applied to deep neural networks. 
Furthermore, in the overparametrized setting, their cost per-iteration in general is at least $\Omega(mn^2)$.

We remark that in the \emph{convex} setting, theoretical guarantees for large-scale second-order algorithms 
have been established (e.g.,\cite{abh17, pw17, mnj16, b15}), but such rigorous analysis in non-convex setting was 
only recently proposed (\cite{cgh+19,zmg19}).  
Our algorithm bears some similarities to the \emph{NewtonSketch} algorithm of~\cite{pw17}, which also incorporates 
sketching into second order Newton methods. A key difference, however, is that the algorithm of~\cite{pw17} 
works only for convex problems, and requires access to $(\nabla^{2}f(x))^{1/2}$ (i.e., the square-root of the Hessian). 
Most importantly, though, \cite{pw17} use the standard (black-box) Sketch-and-Solve 
paradigm to reduce the computational cost, while this approach incurs large computation overhead in our 
non-convex setting. By contrast, we use sketching as a subroutine for fast preconditioning. 
As a by-product, 
\ifdefined\isitcs
in the full version of this paper
\else
in Section \ref{sec:opt} 
\fi
we show how to apply our techniques to give a substantial 
improvement over~\cite{pw17} in the \emph{convex} setting.


The aforementioned works of \cite{zmg19} and \cite{cgh+19} are most similar in spirit to ours. Zhang et al.~\cite{zmg19} 
analyzed the convergence rate of Natural gradient descent algorithms for two-layer (overparametrized) neural 
networks, and showed that the number of iterations is \emph{independent} of the training data size $n$ 
(essentially $\log(1/\eps)$). They also demonstrate similar results for the convergence rate of K-FAC in the 
overparametrized regime, albeit with larger requirement on the width $m$. Another downside of K-FAC is 
the high cost per iteration ($\sim mn^2$). Cai et al.~\cite{cgh+19} analyzed the convergence rate of the 
so-called Gram-Gauss-Newton algorithm for training two-layer (overparametrized) neural network with 
\emph{smooth} activation gates. They proved a quardratic (i.e., doubly-logarithnmic) convergence rate 
in this setting ($\log(\log (1/\eps))$) albeit with $O(mn^2)$ cost per iteration. It is noteworthy that 
this quadratic convergence rate analysis does not readily extend to the more complex and realistic setting 
of ReLU activation gates, which is the focus of our work. \cite{cgh+19} also prove bounds on the 
convergence of `batch GGN', showing that it is possible to reduce the cost-per-iteration to $m$, at the 
price of $O(n^2\log(1 / \eps))$ iterations, for very heavily overparametrized DNNs (currently $m = \Omega(n^{18})$).

\ifdefined\isitcs
\vspace{+2mm}
{\noindent \bf Sketching \ \ }
\else
\paragraph{Sketching}
\fi
The celebrated `Sketch and Solve' (S\&S) paradigm~\cite{cw13} was originally developed to speed up the cost of solving 
linear regression and low-rank approximation problems. This dimensionality-reduction technique has since then been 
widely developed and applied to both convex and non-convex numerical linear algebra problems 
\cite{bwz16,rsw16,wz16,alszz18,bw18,bcw19,ww19,djssw19,swyzz19,song19,bwz20}, as well as machine-learning 
applications \cite{akmmvz17,akmmvz19,lppw20,wz20}.  
The most direct application of the sketch-and-solve technique is overconstrained regression problems, 
where the input is a linear system $[A,b] \in \R^{n \times (d+1)}$ with $n \gg d$, and we aim to find 
an (approximate) solution $\wh{x} \in \R^d$ so as to minimize the residual error $\| A \wh{x} - b \|_2$.

In the classic S\&S paradigm, the underlying regression solver is treated as a \emph{black box}, 
and the computational savings comes from applying it on a smaller \emph{compressed} matrix. 
Since then, sketching (or sampling) has also been used in a non-black-box fashion for speeding-up optimization 
tasks, e.g., as a subroutine for preconditioning~\cite{w14, rt08, st04,  kosz13} 
or fast inverse-maintenance in Linear Programming solvers, semi-definite programming, cutting plane methods, and empirical-risk 
minimization~\cite{cls19,jswz20,jklps20,jlsw20,lsz19}.

\ifdefined\isitcs
\vspace{+2mm}
{\noindent \bf Overparametrization in neural networks \ \ }
\else
\paragraph{Overparametrization in neural networks}
\fi
A long and active line of work in recent deep learning literature has focused on 
obtaining rigorous bounds on the convergence rate of various local-search algorithms 
for optimizing DNNs ~\cite{ll18, dzps19, als19a, als19b, adhlw19, adh+19, sy19, jt20}. 
The breakthrough work of Jacob et al.~\cite{jgh18} and subsequent developments\footnote{For a complete list of references, we refer the readers to \cite{adhlw19,adh+19}.} 
introduced the notion of {\em neural tangent kernels} (NTK), 
implying that for wide enough networks ($m \gtrsim n^4$),  
(stochastic) gradient descent provably converges to an optimal solution, with 
generalization error independent of the number of network parameters.


%% file: tech.tex
\section{Technical Overview}

We now provide a streamlined overview of our main result, Theorem~\ref{thm:second-order-intro}. 
As discussed in the introduction, our algorithm extends to multi-layer ReLU networks 
, though we focus on the  
two-layer case (one-hidden layer), which is the most interesting case where one can indeed hope for linear training time. 

The main, and most expensive step, of the GGN (or natural gradient descent) algorithms \cite{cgh+19, zmg19} is multiplying, 
in each iteration $t$, the \emph{inverse} of the Gram matrix $G_t := J_tJ_t^\top$ 
with the Jacobian matrix $J_t \in \R^{n\times m}$, whose $i$th 
row contains the gradient of the $m=md$ network gates w.r.t the $i$th datapoint $x_i$ (in our case, under ReLU activation).   

Naiively computing $G_t$ would already take $mdn^2$ time, 
however, the \emph{tensor product} structure of the Jacobian $J$ in fact allows to compute $G_t$ 
in $n\cdot \cT_{mat}(m,d,n) \ll mn^2$ time, where $\cT_{mat}(m,d,n)$ is the cost of fast rectangular matrix 
multiplication\cite{w12,l14,gu18}.\footnote{To see this, observe that the kronecker-product structure of $J$ (here $J \in \R^{n\times md}$ can be constructed from an $n \times m$ matrix and an $n \times d$ matrix) allows 
computing $Jh$ for any $h\in \R^{md}$ using fast rectangular matrix multiplication in time $\cT_{mat}(m,d,n)$ 
which is near linear time in the dimension of $J$ and $h$ (that is, $n \times m + n \times d$ for $J$ and $md$ for $h$) so long as $d\leq n^\alpha = n^{0.31}$ \cite{gu18}, hence computing 
$G = JJ^\top$ can be done using $n$ independent invocations of the aforementioned subroutine, yielding 
$n\cdot \cT_{mat}(m,d,n)$ as claimed.} 
Since the Gram-Gauss-Newton (GGN) algorithm requires $O(\log\log 1/\eps)$ 
iterations to converge to an $\eps$-global minimum of the $\ell_2$ loss \cite{cgh+19}, 
this observation yields an $O(n\cdot \cT_{mat}(m,d,n)\lg\lg 1/\eps)$ total time algorithm 
for reducing the training loss below $\eps$. 
While already nontrivial, this is still far from linear running time ($\gg mdn$).

We show how to carry out each Gauss-Newton iteration in time $\tilde{O}(mnd + n^3)$,  
at the price of 
slightly compromising the number of iterations to $O(\lg 1/\eps)$, which is inconsequential 
for the natural regime of constant dimension $d$ and constant $\eps$\footnote{We also remark that this slowdown in the 
convergence rate is also a consequence of a direct extension of the analysis in 
\cite{cgh+19} to ReLU activation functions.}. Our first key step is to reformulate the Gauss-Newton 
iteration (multiplying $G_t^{-1}$ by the error vector) as an \emph{$\ell_2$-regression problem}: 
\begin{equation}\label{eq-regression}
\min_{g_t}\| J_t J_t^{\top} g_t - (f_t - y)\|_2 
\end{equation} 
where $(f_t-y)$ is the training error with respect to the network's output and the training labels $y$.  
Since the Gauss-Newton method is robust to small perturbation errors (essentially \cite{v89_lp,v89_cp}), 
our analysis shows that 
it is sufficient to find an approximate solution $g'_t$ such that $J_t^\top g'_t$  satisfies 
\begin{equation}\label{eq-regression_gamma}
\|J_t J_t^{\top} g_t' - y\|_{2} \leq \gamma \|y\|_2,\;\; \text{for} \;\;\;  \gamma \approx 1/n .  
\end{equation} 
The benefit of this reformulation is that it allows to use \emph{linear sketching} to 
first compress the linear system, significantly reducing the dimension of the optimization problem and 
thereby the cost of finding a solution, at the price of a small error in the found solution 
(this is the essence of the \emph{sketch-and-solve} paradigm \cite{cw13}). 
Indeed, a (variation of) the \emph{Fast-JL} sketch \cite{ac06,ldfu13} guarantees that we can multiply 
the  matrix $J_t^\top \in \R^{m\times n}$ 
by a much smaller $\tilde{O}(n/\delta^2) \times m$ matrix $S$, 
such that (i) the multiplication takes near-linear time $\tilde{O}(mn)$ time (using the FFT algorithm), 
and (ii) $SJ^\top_t$ is a $\delta$-spectral approximation of $J_t^\top$ 
(i.e., $\|J_tS^{\top}SJ_t^{\top}x\|_2 = (1 \pm \delta)\|G_t x\|_2$ for every $x$). Since both computing and 
inverting the matrix $\tilde{G}_t := J_t S^{\top} S J_t^{\top}$ takes $\tilde{O}(n^3/\delta^2)$ 
time, the overall cost of finding a $\delta$-approximate solution to the regression problem becomes at most 
$\tilde{O}(mn+n^3/\delta^2)$. Alas, as noted in Equation \eqref{eq-regression_gamma},   
the approximation error of the found solution must be 
\emph{polynomially small} $\gamma\sim 1/n$ in order to guarantee the desired convergence 
rate (i.e., constant decrease in training error per iteration).   
This means that we must set $\delta \sim \gamma \sim 1/n$, hence the cost of the 
naiive ``sketch-and-solve'' algorithm 
would be at least $\tilde{O}(n^3/\delta^2) = \tilde{O}(n^5)$,
which is a prohibitively large overhead in both theory and practice (and in particular, 
no longer yields linear runtime whenever $m \ll n^4$ which is the current best overparametrization guarantee \cite{sy19}). 
Since the $O(1/\delta^2)$ dependence 
of the JL embedding is known to be tight in general \cite{ln17}, this means we need to 
take a more clever approach to solve the regression \eqref{eq-regression}. 
This is where our algorithm departs from the naiive sketch-and-solve method, and is the heart of our work.

Our key idea is to use dimension reduction---not to directly invert the compressed matrix---but rather to 
\emph{precondition} it quickly. More precisely, our approach is to use a (conjugate) gradient-descent solver for the 
 regression problem itself, with a fast preconditioning step, ensuring exponentially faster convergence 
 to very high (polynomially small) accuracy. Indeed, conjugate gradient descent is guaranteed to find a $\gamma$-approximate 
 solution to a regression problem $\min_x\|Ax-b\|_2$ in 
 $O( \sqrt{\kappa} \log( 1 / \gamma ) )$ iterations, 
 where $\kappa(A)$ is the \emph{condition number} of $A$ (i.e., the ratio of maximum to minimum eigenvalue). 
Therefore, if we can ensure that $\kappa(G_t)$ 
is small, then we can $\gamma$-solve the regression problem 
in $\sim mn\lg(1/\gamma) = \tilde{O}(mn)$ time, since the per-iteration cost of first-order 
SGD is linear ($\sim mn$). 

The crucial advantage of our approach is that it \emph{decouples} the sketching error from the final 
precision of the regression problem: Unlike the usual `sketch-and-solve' method, where 
the sketching error $\delta$ directly affects the overall precision of the solution to \eqref{eq-regression_gamma}, 
here $\delta$ only 
affects the \emph{quality of the preconditioner} (i.e., the ratio of max/min singular values of the sketch $\tilde{G}_t$), 
hence it suffices to take a \emph{constant} sketching error $\delta = 0.1$ (say), while 
letting the SGD deal with the final precision (at it has logarithmic dependence on $\gamma$). 
\ifdefined\isitcs
\else
See Lemma \ref{lem:fast-regression} for the formal details. 
\fi

Indeed, by setting the sketching error to $\delta = 0.1$ (say), 
the resulting matrix $\tilde{G}_t = J_t S^{\top} S J_t^{\top}$ 
is small enough ($n\times \tilde{O}(n)$) that we can afford running a standard 
(QR) algorithm to precondition it, at another $\tilde{O}(n^3)$ cost per iteration. 
The output of this step is a matrix $\tilde{G}_t' := \mathsf{Prec}( \tilde{G}_t )$ with a \emph{constant} 
condition number $\kappa(\tilde{G}_t')$ which preserves $\tilde{G}_t' x \approx_{\ell_2} \tilde{G}_t$ 
up to $(1 \pm \delta)^2$ relative error. 
At this point, we can run a (conjugate) gradient descent algorithm, which is guaranteed to find a 
$\gamma \approx 1/n$ approximate solution to \eqref{eq-regression} 
in time $\tilde{O}(( m n \lg((1 + \delta)/\gamma)+n^3)$, as desired. 

We remark that, by definition, the preconditioning step (on the JL sketch) does \emph{not} preserve 
the eigen-spectrum of $G_t$, which is in fact necessary to guarantee the fast convergence of 
the Gauss-Newton iteration 
\ifdefined\isitcs
\else
(see Lemma \ref{lem:initialization-eigenvalue})
\fi
. The point is that this 
preconditioning step is only preformed as a \emph{local subroutine} so as to solve the regression 
problem, and does \emph{not} affect the convergence rate of the outer loop.





%% file: preli.tex
\section{Preliminaries}
\subsection{Model and Problem Setup}
We denote by $n$ the number of data points in the training batch, and by $d$ the data dimension/feature-space 
(i.e., $x_i \in \R^d$).
We denote by $m$ the \emph{width} of neural network, and by $L$ the number of layers and by $M$ the number of parameters.
We assume the data has been normalized, i.e., $\|x\|_2 = 1$.
 We begin with the  
two-layer neural network in the following section, and then extend to multilayer networks. Consider a two-layer ReLU activated neural network with $m$ neurons in the (single) hidden layer:
\begin{align*}
f (W,x,a) = \frac{1}{ \sqrt{m} } \sum_{r=1}^m a_r \phi ( w_r^\top x ) ,
\end{align*}
where $x \in \R^d$ is the input, $w_1, \cdots, w_m \in \R^d$ are weight vectors in the first layer, $a_1, \cdots, a_m \in \R$ are weights in the second layer.  For simplicity, we consider $a \in \{-1,+1\}^m$ is fixed over all the iterations, this is natural in deep learning theory \cite{ll18,dzps19,als19a,all19,sy19}.
Recall the ReLU function $\phi(x)=\max\{x,0\}$.
Therefore for $r\in [m]$,
we have
\begin{align}\label{eq:relu_derivative}
\frac{\partial f (W,x,a)}{\partial w_r}=\frac{1}{ \sqrt{m} } a_r x{\bf 1}_{ w_r^\top x \geq 0 }.
\end{align}
%
Given $n$ input data points $(x_1 , y_1), (x_2 , y_2) , \cdots (x_n, y_n) \in \R^{d} \times \R$. We define the objective function $\mathcal{L}$
as follows
\begin{align*}
\mathcal{L} (W) = \frac{1}{2} \sum_{i=1}^n ( y_i - f (W,x_i,a) )^2 .
\end{align*}
%
%
%
We can compute the gradient of $\mathcal{L}$ in terms of $w_r$ 
\begin{align}\label{eq:gradient}
\frac{ \partial \mathcal{L}(W) }{ \partial w_r } = \frac{1}{ \sqrt{m} } \sum_{i=1}^n ( f(W,x_i,a) - y_i ) a_r x_i {\bf 1}_{ w_r^\top x_i \geq 0 }.
\end{align}
%
%
%
%
We define the prediction function $f_t : \R^{d \times n} \rightarrow \R^n $ at time $t$ as follow 
\begin{align*}
f_t  = 
\begin{bmatrix}
\frac{1}{ \sqrt{m} } \sum_{r=1}^m a_r \cdot \phi( \langle w_r(t), x_1 \rangle ) \\
\frac{1}{ \sqrt{m} } \sum_{r=1}^m a_r \cdot \phi( \langle w_r(t), x_2 \rangle ) \\
\vdots \\
\frac{1}{ \sqrt{m} } \sum_{r=1}^m a_r \cdot \phi( \langle w_r(t), x_n \rangle ) \\
\end{bmatrix}
\end{align*}
where $W_t = [ w_1(t)^\top , w_2(t)^\top, \cdots, w_m(t)^\top ]^\top \in \R^{md}$ and $X = [x_1, x_2 ,\cdots, x_n] \in \R^{d \times n}$ .

For each time $t$, the Jacobian matrix $J \in \R^{n \times md}$ is defined via the following formulation:
\begin{align*}
J_t = \frac{1}{ \sqrt{m} }
\left[
\begin{matrix}
a_1 x_1^{\top} {\bf 1}_{ \langle w_1(t) , x_1 \rangle \geq 0 } & a_2 x_1^{\top} {\bf 1}_{ \langle w_2(t) , x_1 \rangle \geq 0 } &  \cdots  & a_m x_1^{\top} {\bf 1}_{ \langle w_m(t) , x_1 \rangle \geq 0 }  \\
a_1 x_2^{\top} {\bf 1}_{ \langle w_1(t) , x_2 \rangle \geq 0 } & a_2 x_2^{\top} {\bf 1}_{ \langle w_2(t) , x_2 \rangle \geq 0 } &  \cdots  & a_m x_2^{\top} {\bf 1}_{ \langle w_m(t) , x_2 \rangle \geq 0 }  \\
\vdots & \vdots & \ddots & \vdots \\
a_1 x_n^{\top} {\bf 1}_{ \langle w_1(t) , x_n \rangle \geq 0 } & a_2 x_n^{\top} {\bf 1}_{ \langle w_2(t), x_n \rangle \geq 0 } & \ldots & a_m x_n^{\top} {\bf 1}_{ \langle w_m(t) , x_n \rangle \geq 0 } \\
\end{matrix}
\right].
\end{align*}
The Gram matrix $G_t$ is defined as $G_t = J_t J_t^{\top}$, whose $(i, j)$-th entry is $\left\langle \frac{f(W_t, x_i)}{\partial W}, \frac{f(W_t, x_j)}{\partial W} \right\rangle$. 
The crucial observation of~\cite{jgh18, dzps19} is that the asymptotic of the Gram matrix equals a positive semidefinite kernel matrix $K \in \R^{n \times n}$, where
\begin{align}
\label{eq:kernel}
K(x_i, x_j) = \E_{w\in \mathcal{N}(0, 1)}\left[x_{i}^{\top}x_{j}\textbf{1}_{\langle w, x_i \rangle \geq 0, \langle w, x_j \rangle \geq 0} \right].
\end{align}

\begin{assumption}
	\label{asp:lambda}
	We assume the least eigenvalue $\lambda$ of the kernel matrix $K$ defined in Eq.~\eqref{eq:kernel} satisfies $\lambda > 0$.
\end{assumption}






\subsection{Subspace embedding}

Subspace embedding was first introduced by Sarl\'{o}s \cite{s06}, it has been extensively used in numerical linear algebra field over the last decade \cite{cw13,nn13,bw14,swz19b}. For a more detailed survey, we refer the readers to \cite{w14}. The formal definition is:
\begin{definition}[Approximate subspace embedding, {\sf ASE} \cite{s06}]\label{def:ase1}
	A $(1 \pm \epsilon)$ $\ell_2$-subspace embedding for the column space of an $N \times k$ 
	matrix $A$ is a matrix $S$ for which for all $x \in \R^k$, $\| S A x \|_2^2 = (1 \pm \epsilon) \| A x \|_2^2$.
	Equivalently,
	$
	\| I - U^\top S^\top S U \|_2 \leq \epsilon,
	$ 
where $U$ is an orthonormal basis for the column space of $A$.%
\end{definition}

Combining Fast-JL sketching matrix \cite{ac06,dmm06,t11,dmmw12,ldfu13,psw17} with a classical $\epsilon$-net argument \cite{w14} gives subspace embedding, 
\begin{lemma}[Fast subspace embedding~\cite{ldfu13,w14}]
	\label{fact:subspace-embedding}
	Given a matrix $A \in \R^{N \times k}$ with $N = \poly(k)$, then we can compute a $S \in \R^{k \poly ( \log (k/\delta) )/\eps^2 \times k}$ that gives a 
	subspace embedding of $A$ with probability $1-\delta$, i.e., with probability $1-\delta$, we have :
	\begin{align*}
	\| S A x \|_2 = (1 \pm \eps) \| A x \|_2
	\end{align*}
	holds for any $x \in \R^{n}$, $\| x \|_2 = 1$.
	Moreover, $SA$ can be computed in $O(Nk\cdot \poly\log k)$ time.
\end{lemma}


%% file: result.tex

\section{Our Algorithm}
\label{sec:result}

Our main algorithm is shown in Algorithm~\ref{alg:fast}. We have the following convergence result of our algorithm.
\begin{theorem}
	\label{thm:second-order-result}
	Suppose the width of a ReLU neural network satisfies 
	\begin{align*}
	m = \Omega (\max \{ \lambda^{-4} n^4, \lambda^{-2} n^2d\log(16n/\delta) \} ),
	\end{align*}
	then with probability $1 - \delta$ over the random initialization of neural network and the randomness of the training algorithm, our algorithm (procedure \textsc{FasterTwoLayer} in Algorithm~\ref{alg:fast}) achieves
	\begin{align*}
	\| f_{t+1} - y \|_2 \leq  \frac{1}{2}  \| f_t - y \|_2. 
	\end{align*}
	The computation cost in each iteration is $\tilde{O}(mnd + n^3)$, and the running time for reducing the 
	training loss to $\eps$ is $\tilde{O}((mnd + n^3) \log(1/\eps))$. Using fast matrix-multiplication, the total running time can be further reduced to 
	$\tilde{O}((mnd + n^{\omega}) \log(1/\eps))$.
\end{theorem}

\begin{algorithm}[!h]
	\caption{Faster algorithm for two-layer neural network} 
	\label{alg:fast}
	\begin{algorithmic}[1]
	\Procedure{FasterTwoLayer}{$ $} \Comment{Theorem~\ref{thm:second-order-result}}
		\State $W_0$ is a random Gaussian matrix \Comment{$W_0 \in \R^{md}$}
		\While{$t < T$}
		\State Compute the Jacobian matrix $J_t$ \Comment{$J_t\in \R^{n \times m d}$}
		\State Find an $\eps_0$ approximate solution using Algorithm~\ref{alg:fast-regression} \Comment{$\eps_0 \in (0, \frac{1}{6}\sqrt{\lambda/n}]$}
		\begin{align}
		\label{eq:reg-algo}
		\min_{g_t}\| J_t J_t^{\top} g_t - (f_t - y)\|_2
		\end{align}
		\State Update $W_{t+1} \leftarrow W_{t} - J_t^\top g_t$ \label{line:approx}
		\State $t \leftarrow t + 1$
		\EndWhile
	\EndProcedure
	\end{algorithmic}
\end{algorithm}

The main difference between \cite{cgh+19, zmg19} and our algorithm is that we perform an {\em approximate Newton update} (see line~\ref{line:approx}). 
The crucial observation here is that the Newton method is robust to small loss, thus it suffices to present a fine approximation.
This observation is well-known in the convex optimization but unclear to the non-convex (but overparameterized) neural network setting.
Another crucial observation is that instead of directly approximating the Gram matrix, it is suffices to approximate $(J_tJ_t^{\top})^{-1}g_t = G_t^{-1}g_t$. Intuitively, this follows from
\[
J_t^{\top}g_t \approx J_t(J_tJ_t^{\top})^{-1}(f_t - y) = (J_t^{\top}J_t)^{\dagger}J_t(f_t - y),
\]
where $(J_t^{\top}J_t)^{\dagger}$ denotes the pseudo-inverse of $J_t^{\top}J_t$ and the last term is exactly the Newton update.
This observation allows us to formulate the problem a regression problem  (see Eq.~\eqref{eq:reg-algo}), on which we can introduce techniques from {\em randomize linear algebra} and develop fast algorithm that solves it in near linear time.

\subsection{Fast regression solver}
\begin{algorithm}[!h]
	\caption{Fast regression} 
	\label{alg:fast-regression}
	\begin{algorithmic}[1]
	\Procedure{FastRegression}{$ A, \eps$} \Comment{Lemma~\ref{lem:fast-regression-main}}
		\State \Comment{$A\in \R^{N \times k}$ is a full rank matrix, $\eps \in (0,1/2)$ is the desired precision} 
		\State Compute a subspace embedding $SA$  \Comment{$S \in \R^{k\poly(\log k) \times k}$} \label{line:regression1}
		\State Compute $R$ such that $SAR$ orthonormal columns via QR decomposition 		\Comment{$R\in \R^{k \times k}$}\label{line:regression2}
		\State $z_0 \leftarrow \vec{0} \in \R^{k}$
		\While{$\|A^{\top}ARz_t - y\|_2 \geq \eps$}\label{line:regression3}
		\State $z_{t+1} \leftarrow z_{t} - (R^{\top}A^{\top}AR)^{\top}(R^{\top}A^{\top}ARz_t - R^{\top}y)$ 
		\EndWhile\label{line:regression4}\\
		\Return $R z_t$
	\EndProcedure
	\end{algorithmic}
\end{algorithm}
The core component of our algorithm is a fast regression solver (shown in Algorithm~\ref{alg:fast-regression}). 
The regression solver provides an approximate solution to $\min_{x}\|A^{\top}Ax - y\|$ where $A \in \R^{N \times k}$ ($N \gg k$). 
We perform preconditioning on the matrix of $A^{\top}A$ (line~\ref{line:regression1} -- \ref{line:regression2}) and use gradient descent to derive an approximation solution (line~\ref{line:regression3} -- \ref{line:regression4}). 
\begin{lemma}\label{lem:fast-regression-main}
Let $N = \Omega (k \poly(\log k) )$. Given a matrix $A \in \R^{N \times k}$, let $\kappa$ denote the condition 
	number of $A$ \footnote{$\kappa= \sigma_{\max}(A) / \sigma_{\min}(A)$}, consider the following regression problem
	\begin{align}
	\label{eq:reg2}
	\min_{x \in \R^k} \| A^{\top} A x - y \|_{2}.
	\end{align}
	Using procedure \textsc{FastRegression} (in Algorithm~\ref{alg:fast-regression}), with probability $1 - \delta$, we can compute an $\eps$-approximate solution $x'$ satisfying 
	\begin{align*}
	\|A^{\top}Ax' - y\|_{2} \leq \eps \|y\|_2
	\end{align*}
	in $\tilde{O}\left(Nk\log(\kappa/\eps)+ k^3\right)$ time. 
\end{lemma}

\ifdefined\isitcs
\vspace{+2mm}
{\noindent \bf Speedup in Convex Optimization \ \ }
\else
\paragraph{Speedup in Convex Optimization}
\fi
It should come as no surprise that our techniques can help accelerating a broad class of solvers in \emph{convex optimization} 
problems as well. In the full version of this paper, we elaborate on this application, and in particular show how 
our technique improves the runtime of the ``Newton-Sketch'' algorithm of ~\cite{pw17}.


%% file: concl.tex
\section{Conclusion and Open Problems}\label{sec_concl}

Our work provides a computationally-efficient (near-linear time) second-order algorithm for training  
sufficiently overparametrized two-layer neural network, overcoming the drawbacks of traditional first-order gradient algorithms.  
Our main technical contribution is developing a \emph{faster regression solver} which uses 
linear sketching for fast preconditioning (in time independent of the network width). 
As such, our work demonstrates that the toolbox of randomized linear algebra can 
substantially reduce the computational cost of second-order methods in 
\emph{non-convex optimization}, 
and not just in the convex setting for which it was originally developed 
(e.g., \cite{pw17, w14,cls19, jswz20, jklps20, jlsw20, lsz19}). 


Finally, we remark that, while the running time of our algorithm is $\tilde{O}(Mn + n^3)$ 
(or $O(Mn + n^{\omega})$ using FMM), it is no longer (near) linear 
for networks with parameters $M \leq n^2$ (resp. $M \lesssim n^{\omega-1}$). 
While it is widely believed that $\omega=2$ \cite{cksu05}, FMM algorithms are impractical at present, and it  
would therefore be very interesting to improve the extra additive term from $n^3$ to $n^{2+o(1)}$  (which 
seems best possible for dense $n\times n$ matrices), or even to $n^{3- \eps}$ using a practically viable algorithm.  
Faster preconditioners seem key to this avenue.

%% file: app.tex
\section{Appendix}
\label{sec:appendix}
{\bf Organization} The Appendix is organized as follows. 
Section~\ref{sec:appendix} contains notations and some basic facts. 
In Section~\ref{sec:fast-regression} we present the fast regression solver.
In Section~\ref{sec:two-layer} we prove our main result for two-layer ReLU networks. 
Finally, in Section~\ref{sec:opt} we show that 
our optimization framework can obtain acceleration in classic convex 
optimization setting, improve over~\cite{pw17}.

\subsection{Notation}
\label{sec:notation}
For a vector $x\in \R^n$, we use $\|x\|_2$ to denote the $\ell_2$ norm, i.e., $\|x\|_2 = ( \sum_{i=1}^{n}x_i^2 )^{1/2}$. 
We use $\| x \|_1$ to denote its $\ell_1$ norm, $\| x \|_{\infty}$ to denote its $\ell_{\infty}$ norm. 
For a matrix $A$, we use $\| A \|$ to denote its spectral norm, i.e., $\| A \| = \max_{ \| x \|_2 = 1 } \| A x \|_2$. 
We use $\| A \|_F$ to denote the Frobenius norm, i.e., $\| A \|_{F} = ( \sum_{i=1}^{m} \sum_{j=1}^{n} A_{i, j}^2 )^{1/2}$. 
We $A^{\top}$ to denote the transpose of matrix $A$.
We use $\sigma_{\min} ( A ) $ to denote the minimum singular value of $A$, i.e., $\sigma_{\min} = \min_{ \| x \|_2 = 1 } \| A x \|_2$. 
We define $\sigma_{\max}$ to be the maximum singular value and we have $\sigma_{\max}(A) = \|A\|$.
We use $\kappa ( A )$ to denote the condition number of $A$, i.e., $\kappa ( A ) = \sigma_{\max} (A)/ \sigma_{\min}(A)$.
We write $x = y \pm \eps$ if $x \in [y - \eps, y + \eps]$.
For a positive semidefinite (PSD) matrix $A$, we sometimes use $\lambda_{\min}(A)$ (resp. $\lambda_{\max}(A)$) to denote the minimum (resp. maximum) eigenvalue of $A$.

\subsection{Probability Tools}\label{sec:app_prob}


\begin{lemma}[Chernoff bound \cite{c52}]\label{lem:chernoff}
Let $X = \sum_{i=1}^n X_i$, where $X_i=1$ with probability $p_i$ and $X_i = 0$ with probability $1-p_i$, and all $X_i$ are independent. Let $\mu = \E[X] = \sum_{i=1}^n p_i$. Then \\
1. $ \Pr[ X \geq (1+\delta) \mu ] \leq \exp ( - \delta^2 \mu / 3 ) $, $\forall \delta > 0$ ; \\
2. $ \Pr[ X \leq (1-\delta) \mu ] \leq \exp ( - \delta^2 \mu / 2 ) $, $\forall 0 < \delta < 1$. 
\end{lemma}

\begin{lemma}[Hoeffding bound \cite{h63}]\label{lem:hoeffding}
Let $X_1, \cdots, X_n$ denote $n$ independent bounded variables in $[a_i,b_i]$. Let $X= \sum_{i=1}^n X_i$, then we have
\begin{align*}
\Pr[ | X - \E[X] | \geq t ] \leq 2\exp \left( - \frac{2t^2}{ \sum_{i=1}^n (b_i - a_i)^2 } \right).
\end{align*}
\end{lemma}



\begin{lemma}[folklore]
\label{lem:anti_gaussian}
Let $X \sim {\cal N}(0,\sigma^2)$,
that is,
the probability density function of $X$ is given by $\phi(x)=\frac 1 {\sqrt{2\pi\sigma^2}}e^{-\frac {x^2} {2\sigma^2} }$.
Then
\begin{align*}
    \Pr[|X|\leq t] \leq \frac{4}{5} \frac{t}{\sigma}. 
\end{align*}
\end{lemma}

\subsection{Basic Facts}
\label{sec:fact}

\begin{fact}
	\label{fact:kappa}
	For any two matrices $A, B$, $\kappa(B) \leq \kappa(AB) \kappa(A)$.
\end{fact}
\begin{proof}
	We know for any $\|x\|_2 = 1$,
	\begin{align*}
	\sigma_{\min}(A) \| B x \|_2\leq \| A B x \|_2 \leq \sigma_{\max} ( A B ) \|x\|_2 = \sigma_{\max} ( A B ) .
	\end{align*} 
	Hence we have $\sigma_{\max}( B ) \leq \sigma_{\max} ( A B ) /\sigma_{\min}(A)$. Similarly, we have
	\begin{align*}
	\sigma_{\max}(A)\| B x \|_2 \geq \| A B x \|_2 \geq \sigma_{\min}( A B ) \|x\|_2 = \sigma_{\min} ( A B )
	\end{align*}
	i.e., $\sigma_{\min} ( B ) \geq \sigma_{\min} ( A B ) /\sigma_{\max}(A)$.
	Thus we conclude
	\begin{align*}
	\kappa( B ) \leq \kappa( A B ) \kappa(A).
	\end{align*}
\end{proof}




%% file: regression.tex
\section{Fast regression solver}
\label{sec:fast-regression}


\begin{lemma}[Formal version of Lemma~\ref{lem:fast-regression-main}]
	\label{lem:fast-regression}
	Given a matrix $A \in \R^{N \times k}$ ($N \geq k\poly(\log k)$), let $\kappa$ denote the condition 
	number of $A$ \footnote{$\kappa= \sigma_{\max}(A) / \sigma_{\min}(A)$}, consider the following regression problem
	\begin{align}
	\label{eq:reg1}
	\min_{x \in \R^k} \| A^{\top} A x - y \|_{2}.
	\end{align}
	We can compute an $\eps$-approximate solution $x'$ satisfying
	\begin{align*}
	\|A^{\top}Ax' - y\|_{2} \leq \eps \|y\|_2
	\end{align*}
	in $\tilde{O}\left(Nk\log(\kappa/\eps)+ k^3\right)$ time. 
	Using fast matrix-multiplication, the total running time can be further reduced to 
	$\tilde{O}((mnd + n^{\omega}) \log(1/\eps))$. \Jan{added the last sentence so we can state the omega complexity in thm 4.1}
\end{lemma}

\begin{proof}
	Using lemma~\ref{fact:subspace-embedding}, let $S\in \R^{k\poly(\log k/\delta)/\eps_0^2 \times N}$ be a subspace embedding of $A$, with probability $1 - \delta$, the following holds for any $x\in \R^{k}$
	\begin{align*}
	\|SAx\|_2 = (1 \pm \eps_0)\|Ax\|_2.
	\end{align*}
	
	Suppose $R\in \R^{k\times k}$ is computed so that $SAR$ has orthonormal columns, e.g., via QR decomposition. 
	We use $R$ as a preconditioner for matrix $A$. Formally, for any $x\in \R^{n}$ satisfying $\|x\|_2=1$, we have 
	\begin{align}
	\label{eq:sketch-precondition2}
	\|ARx\|_2 = (1 \pm \eps_0)\|SARx\|_2 = (1 \pm \eps_0).
	\end{align}
	
	Hence, we know for any $\|x\|_2 = 1$, 
	\begin{align*}
	(1 - \eps_0)^2 \leq \|R^{\top}A^{\top}ARx\|_2 \leq (1 + \eps_0)^2.
	\end{align*}
	
	We choose $\eps_0 = 0.1$, and consider the regression problem
	\begin{align}
	\label{eq:reg}
	\min_{z \in \R^n} \| R^{\top} A^{\top} A R z - R^{\top} y \|_{2}.
	\end{align}
	By lemma~\ref{lem:regular-regression}, using gradient descent, after $t = \log(1 / \eps)$ iterations, we can find $z_t$ satisfying
	\begin{align}
	\label{eq:sketch-precondition1}
	\|R^{\top}A^{\top}AR(z_t - z^{\star}) \|_2 \leq \eps\|R^{\top}A^{\top}AR(z_0 - z^{\star})\|_2,
	\end{align}
	where $z^{\star} =  (R^{\top}A^{\top}AR)^{-1}R^{\top}y$ is the optimal solution to Eq.~\eqref{eq:reg}.
	We are going to show that $x_t = Rz_t$ is an $2\kappa\eps$-approximate solution to the original regression problem~\eqref{eq:reg1}, i.e., 
	\begin{align*}
	\|A^{\top}Ax_t - y\|_2 \leq \kappa\eps \|y\|_2
	\end{align*}
	Plugging $z_0 = 0$ into Eq.~\eqref{eq:sketch-precondition1}, we get
	\begin{align}
	\label{eq:reg4}
	\|R^{\top}A^{\top}Ax_t - R^{\top}y \|_2 \leq \eps\|R^{\top}y\| \leq \eps\cdot\sigma_{\max}(R^{\top})\|y\|_2
	\end{align}
	On the other hand, we have
	\begin{align}
	\label{eq:reg5}
	\|R^{\top}A^{\top}Ax_t - R^{\top}y \|_2 = \|R^{\top} (A^{\top}Ax_t - y) \|_2 \geq \sigma_{\min}(R^{\top})\|A^{\top}Ax_t - y\|_2.
	\end{align}


	Putting it all together, we have
	\begin{align*}
	\|A^{\top}Ax_t - y\|_2 
	\leq \eps \kappa(R^{\top}) \|y\|_2
	= \eps\kappa(R)\|y\|_2 
	\leq \eps \kappa (AR)\kappa(A)\|y\|_2 
	\leq 2\eps \kappa(A)\|y\|_2 
	\end{align*}
	where the first step follows from Eq.~\eqref{eq:reg4}~\eqref{eq:reg5}, 
	the second step follows from $R$ is a square matrix and thus $\kappa(R) = \kappa(R^{\top})$, the third step follows from Fact~\ref{fact:kappa} and the last step follows from Eq.~\eqref{eq:sketch-precondition2}.

	For the running time, the preconditioning time is $\tilde{O}(Nk + k^3)$, the number of iteration for gradient desent is $\log(\kappa/\eps)$, the running time per iteration is $\tilde{O}(Nk)$, thus the total running time is 
	\begin{align*}
	\tilde{O}\left(Nk\log(\kappa/\eps) + k^3\right).
	\end{align*}
	\Jan{Reviewer didn't know QR can be done in omega instead of qubic time. So I added a citation.}
	The preconditioning can be reduced to $\tilde{O}(Nk + k^\omega)$
	when using fast matrix multiplication to compute the QR decomposition of $SA$ \cite{DemmelDH07}.
\end{proof}

\begin{lemma}
	\label{lem:regular-regression}
	Consider the the regression problem 
	\begin{align*}
	\min_{x}\|Bx - y\|_2^2.
	\end{align*}
	Suppose $B$ is a PSD matrix with $\frac{3}{4} \leq\|Bx\|_2 \leq \frac{5}{4}$ holds for all $\|x\|_2 = 1$. Using gradient descent, after $t$ iterations, we obtain
	\begin{align*}
	\|B(x_t - x^{\star})\|_2 \leq c^{t}\|B(x_0 - x^{\star})\|_2
	\end{align*}
	for some constant $c \in (0,0.9]$.
\end{lemma}
\begin{proof}
	The gradient at time $t$ is $B^{\top}(Bx_t - y)$ and $x_{t + 1} = x_t - B^{\top}(Bx_t - y)$, thus we have
	\begin{align*}
	\|B x_{t+1} - B x^{\star} \|_2 
	= & ~ \|B ( x_{t} - B^{\top}(B x_t - y) ) - B x^{\star} \|_2\\
	= & ~ \|B ( x_t - x^{\star} ) - BB^{\top} B x_t + B B^{\top} B x^{\star} \|_2\\
	= & ~ \|(I - BB^{\top}) B ( x_t - x^{\star} ) \|_2\\
	\leq & ~ \| I - BB^{\top} \| \cdot \|B ( x_t - x^{\star} ) \|_2\\
	\leq & ~ \frac{9}{16}  \|B ( x_t - x^{\star} ) \|_2
	\end{align*}
	The second step follows from $B^{\top}Bx^{\star} = B^{\top}y$.  The last step follows from the eigenvalue of $BB^{\top}$ belongs to $[\frac{9}{16}, \frac{25}{16}]$ by our assumption. Thus we complete the proof.\qedhere

\end{proof}

%% file: two.tex
\section{Our Algorithm}
\label{sec:two-layer}

We delicate to prove the following result in this section, which is essentially Theorem~\ref{thm:second-order-result}.
\begin{theorem}[Formal version of Theorem~\ref{thm:second-order-result}]
	\label{thm:second-order}
	Suppose the width of the neural network satisfies
	$m = \Omega (\max \{ \lambda^{-4} n^4, \lambda^{-2} n^2d\log(16n/\delta) \} )$,
	then with probability $1 - \delta$ over the random initialization of neural network and the randomness of the algorithm, our algorithm achieves
	\begin{align*}
	\| f_{t+1} - y \|_2 \leq \frac{1}{2} \| f_t - y \|_2.
	\end{align*}
	The computation cost in each iteration is $\tilde{O}(mnd + n^3)$, and the running time for reducing the training loss to $\eps$ is $\tilde{O}((mnd + n^3) \log(1/\eps))$. Using fast matrix multiplication, the running time is $\tilde{O}((mnd + n^\omega) \log(1/\eps))$.
\end{theorem}

The follow lemmas are standard in literature.
\begin{lemma}[Bounds on initialization, Lemma 2 in~\cite{cgh+19}]
	\label{lem:initialization}
	Suppose $m = \Omega( d \log (n/ \delta ) )$, then with probability $1 - \delta$, we have the following
	\begin{itemize}
		\item $f ( W , x_i ) = O(1)$, for $i \in [n]$. 
		\item $\|J_{W_0, x_i}\|_{F} = O(1)$, for $i \in [n]$.
	\end{itemize}
\end{lemma}

\begin{lemma}[Bounds on the least eigenvalue at intialization, Lemma 3 in~\cite{cgh+19}]
	\label{lem:initialization-eigenvalue}
	Suppose 
	$m = \Omega ( \lambda^{-2} n^2\log(n/\delta) )$, then with probability at least $1- \delta$, we have
	\begin{align*}
	\lambda_{\min}(G_0) \geq \frac{3}{4}\lambda.
	\end{align*}
\end{lemma}

When weights  do not change very much, we have
\begin{lemma}
	\label{lem:small-move}
	Suppose $R \geq 1$ and $m = \tilde{\Omega}(n^2 R^2)$. With probability at least $1 - \delta$ over the random initialization of $W_0$, the following holds for {\em any} set of weights $w_1, \ldots w_m \in \R^{d}$ satisfying $\max_{r \in [m]} \|w_r - w_r(0)\|_2 \leq R/\sqrt{m}$,
	\begin{itemize}
		\item $\|W - W_0\| = O(R)$,
		\item $\|J_{W, x_i} - J_{W_0, x_i}\|_{2} = \tilde{O}  ( {R^{1/2}} / {m^{1/4}} )$ and $\|J_{W} - J_{W_0}\|_{F} = \tilde{O} ( {n^{1/2} R^{1/2}} / {m^{1/4}} )$,
		\item $\|J_{W}\|_{F} = O(\sqrt{n})$,
	\end{itemize}
\end{lemma}

\begin{proof}
	(1) The first claim follows from 
	\begin{align*}
	\|W - W_0\| \leq \|W - W_{0}\|_{F}  = \Big( \sum_{r=1}^{m} \|w_r - w_r(0)\|_2^2 \Big)^{1/2} \leq \sqrt{m}\cdot R/\sqrt{m} = R.
	\end{align*}
	(2) For the second claim, we have for any $i \in [n]$
	\begin{align}
	\|J_{W, x_i} - J_{W_0, x_i}\|^2 = & ~\frac{1}{m}\sum_{r=1}^{m}a_r^2\cdot \|x_r\|_2^2\cdot |\textbf{1}_{\langle w_r, x_i \rangle\geq 0} - \textbf{1}_{\langle w_r(0), x_i \rangle \geq 0}|^2 \notag \\
	= & ~ \frac{1}{m}\sum_{r=1}^{m}|\textbf{1}_{\langle w_r, x_i \rangle\geq 0} - \textbf{1}_{\langle w_r(0), x_i \rangle \geq 0}|.\label{eq:standard1}
	\end{align}
	The second equality follows from $a_r \in \{-1, 1\}$, $\|x_i\|_2 =1$ and
	\begin{align}
	\label{eq:standard2}
	s_{i, r} := |\textbf{1}_{\langle w_r, x_i \rangle\geq 0} - \textbf{1}_{\langle w_r(0), x_i \rangle \geq 0}| \in \{0, 1\}.
	\end{align}
	We define the event $A_{i, r}$ as
	\begin{align*}
	A_{i, r} = \left\{\exists \tilde{w} ~:~ \|\tilde{w} - w_r(0)\| \leq R/\sqrt{m}, ~~~ \textbf{1}_{\langle \tilde{w}, x_i \rangle\geq 0} \neq \textbf{1}_{\langle w_r(0), x_i \rangle \geq 0} \right\}.
	\end{align*}
	It is easy to see $A_{i, r}$ happens if and only if $w_r(0)^{\top}x_i \in [-R/\sqrt{m}, R/\sqrt{m}]$. By the anticoncentration of Gaussian (see Lemma~\ref{lem:anti_gaussian}), we have $\E[s_{i, r}] = \Pr[A_{i, r}] \leq \frac{4}{5}R/\sqrt{m}$.
	Thus we have
	\begin{align}
	\Pr\left[\sum_{i=1}^{m} s_{i, r} \geq  ( t + 4 / 5 ) R\sqrt{m} \right] 
	\leq & ~ \Pr\left[\sum_{i=1}^{m} (s_{i, r} - \E [ s_{i, r} ] ) \geq  t R\sqrt{m} \right] \notag \\
	\leq & ~ 2\exp\left( -\frac{2t^2R^2 m}{m} \right) \notag \\
	= & ~ 2\exp(-t^2R^2) \notag \\
	\leq & ~ 2\exp(-t^2). \label{eq:standard3}
	\end{align}
	holds for any $t > 0$.
	The second inequality comes from the Hoeffding bound (see Lemma~\ref{lem:hoeffding}), the last inequality comes from $R > 1$. Taking $t = 2\log(n / \delta)$ and using union bound over $i$, with probability $1 - \delta$, we have
	\begin{align*}
	\|J_{W, x_i} - J_{W_0, x_i}\|_{2}^2 = \frac{1}{m}\sum_{r=1}^{m}s_{i, r} \leq \frac{1}{m} \cdot 2\log (n/\delta) R\sqrt{m} = \tilde{O}(R/\sqrt{m})
 	\end{align*}
 	holds for all $i \in [n]$. The first equality comes from Eq.~\eqref{eq:standard1} and Eq.~\eqref{eq:standard2}, the second inequality comes from Eq.~\eqref{eq:standard3}. Thus we conclude with 
 	\begin{align*}
 	\|J_{W, x_i} - J_{W_0, x_i}\|_{2} = \tilde{O} ( R^{1/2} / m^{1/4} ) \text{~~~and~~~}\|J_{W} - J_{W_0}\|_{F} = \tilde{O} ( n^{1/2} R^{1/2}  / m^{1/4} ).
 	\end{align*}
 	
 	(3) The thrid claim follows from
 	\begin{align*}
 	\|J_{W}\|_{F} \leq \|J_{W_0}\|_{F} + \|J_{W} - J_{W_0}\|_{F} \leq O(\sqrt{n}) + \tilde{O} ( n^{1/2}R^{1/2} / m^{1/4} ) = O(\sqrt{n}).
 	\end{align*}
 	The second inequality follows from $m = \tilde{\Omega}(R^2n^2)$.
	
\end{proof}

\begin{lemma}[Bounds on the least eigenvalue during optimization, Lemma 4.2 in~\cite{sy19}]
	\label{lem:small-move-eigenvalue}
Suppose $m = \Omega(n^2 R^2\log (n / \delta))$,  with probability at least $1 - \delta $, the following holds for {\em any} set of weights $w_1, \ldots w_m \in \R^{d}$ satisfying $\max_{r \in [m] }\|w_r - w_r(0)\|_2 \leq R/\sqrt{m}$,
\begin{align*}
\|G_W - G_{W_0}\|_{F} \leq \lambda / 2.
\end{align*}
	
\end{lemma}

We now begin the proof of Theorem~\ref{thm:second-order}
\begin{proof}[Proof of Theorem~\ref{thm:second-order}]
	We use induction to prove the following two claims recursively. 
	We take $R \approx n / \lambda$ in the proof.
	\begin{enumerate}
		\item $\| w_r(t) - w_r(0) \|_2 \leq R / \sqrt{m}$ holds for any $r\in [m]$ and $t \geq 0$.
		\item $\|f_{t} - y\|_2 \leq \frac{1}{2} \|f_{t-1} - y\|_2$ holds for any $t \geq 1$.
	\end{enumerate}
	Suppose the above two claims hold up to $t$, we prove they continue to hold for time $t + 1$. The second claim is more delicate, we are going to prove it first and we define
	\begin{align*}
	J_{t, t+1} = \int_{0}^{1} J \Big( (1 - s) W_t + s W_{t+1} \Big) \d s.
	\end{align*}
	Hence, we have
	\begin{align}
	&~\|f_{t+1} - y\|_2 \notag\\
	= &~ \|f_{t}  - y + (f_{t + 1}  - f_{t})\|_2\notag\\
	= &~ \|f_{t}  - y +J_{t, t+1}(W_{t+1} - W_{t})\|_2\notag\\
	= &~ \|f_{t}  - y -J_{t, t+1}J_t^{\top} g_t\|_2\notag\\
	= &~ \|f_{t}  - y -J_{t}J_t^{\top} g_t + J_{t}J_t^{\top} g_t  -J_{t, t+1}J_t^{\top} g_t\|_2 \notag\\
	\leq &~ \|f_{t}  - y - J_t J_t^{\top} g_t\|_2 +   \|(J_t-J_{t, t+1})J_t^{\top} g_t\|_2\notag\\
	\leq &~ \|f_{t}  - y - J_t J_t^{\top} g_t\|_2 +\|(J_t-J_{t, t+1})J_t^{\top} g^{\star}\|_2 +   \|(J_t-J_{t, t+1})J_t^{\top} (g_t - g^{\star})\|_2 \label{eq:second-order3},
	\end{align}
	where we denote $g^{\star} = (J_tJ_t^{\top})^{-1}(f_t - y)$ to be the optimal solution to Eq.~\eqref{eq:reg-algo}. 
	The second step follows from the definiton of $J_{t, t+1}$ and simple calculus. 
	The third step follows from the updating rule of the algorithm.
	
	For the first term of Eq.~\eqref{eq:second-order3}, we have
	\begin{align}
	\label{eq:second-order5}
	\|J_t J_t^{\top}g_t - (f_t - y)\|_2 \leq \frac{1}{6}\|f_t - y\|_2,
	\end{align}
	since $g_t$ is an $\eps_0 (\eps_0 \leq \frac{1}{6})$ approximate solution to regression problem~\eqref{eq:reg-algo}.
	
	For the second term in Eq.~\eqref{eq:second-order3}, we have
	\begin{align}
	\|(J_t-J_{t, t+1})J_t^{\top} g^{\star}\|_2 
	\leq & ~ \|(J_t-J_{t, t+1})\| \cdot \|J_t^{\top} g^{\star}\|_2\notag\\
	 = & ~ \|(J_t-J_{t, t+1})\| \cdot \|J_t^{\top} (J_t J_t^{\top})^{-1} (f_t - y) \|_2\notag\\
	 \leq & ~ \|(J_t-J_{t, t+1})\| \cdot \|J_t^{\top} (J_t J_t^{\top})^{-1} \| \cdot \| (f_t - y)\|_2 \label{eq:second-order9}.
	\end{align}
	We bound these term separately. First,
		\begin{align}
		\|J_t - J_{t, t+1}\|
		\leq & ~ \int_{0}^{1}\| J((1 - s)W_t + sW_{t+1}) - J(W_t) \|\d s \notag\\
		\leq &~ \int_{0}^{1}\left( \| J((1 - s)W_t + sW_{t+1}) - J(W_0)\| + \|J(W_0) - J(W_t) \| \right) \d s \notag\\
		\leq &~ \tilde{O} ( R^{1/2} n^{1/2} / m^{1/4} ). 
		\label{eq:second-order6}
		\end{align}
		The third step follows from the second claim in Lemma~\ref{lem:small-move} and the fact that
		\begin{align*}
		\|(1-s)w_r(t) + sw_r(t+1) - w_0\|_2 \leq &~ (1 - s)\|w_r(t) - w_r(0)\|_2 + s\|w_{r}(t + 1) - w_r(0)\|_2\\
		\leq &~ R /\sqrt{m}.
		\end{align*}
	Furthermore, we have
	\begin{align}
		\|J_t^{\top} (J_t J_t^{\top})^{-1}\| = \frac{1}{\sigma_{\min}(J_t^{\top})} \leq \sqrt{ 2 / \lambda } \label{eq:second-order10}
	\end{align}
	The second inequality follows from $\sigma_{\min}(J_t) = \sqrt{\lambda_{\min}(J_t^{\top}J_t)} \geq \sqrt{ \lambda / 2 }$ (see Lemma~\ref{lem:small-move-eigenvalue}). 

	Combining Eq.~\eqref{eq:second-order9}, \eqref{eq:second-order6} and \eqref{eq:second-order10}, we have
	\begin{align}
	\label{eq:second-order11}
	\|(J_t-J_{t, t+1})J_t^{\top} g^{\star}\|_2  \leq \tilde{O} ( {R^{1/2}\lambda^{-1}n^{1/2}} / {m^{1/4} } )\|f_t - y\|_2 \leq \frac{1}{6}\|f_t - y\|,
	\end{align}
	since $m = \tilde{\Omega}(\lambda^{-4} n^4)$.
	
	For the third term in Eq.~\eqref{eq:second-order3}, we have
	
	\begin{align}
	\label{eq:second-order4}
	\|(J_t-J_{t, t+1})J_t^{\top} (g_t - g^{\star})\|_2  \leq\|J_t - J_{t, t+1}\| \cdot \|J_t^{\top}\| \cdot \|g_t - g^{\star}\|_2.
	\end{align}

	Moreover, one has
	\begin{align}\label{eq:second-order2}
	\frac{\lambda}{2} \|g_t - g^{\star}\|_2 
	\leq & ~ \lambda_{\min}(J_t J_t^{\top}) \|g_t - g^{\star}\|_2 \notag \\
	\leq & ~ \|J_t J_t^{\top}g_t - J_t J^{\top}_tg^{\star}\|_2 \notag \\
	= & ~ \|J_t J_t^{\top}g_t - (f_t - y)\|_2 \notag \\
	\leq & ~ \sqrt{ \lambda / n } \cdot \|f_t - y\|_2 .
	\end{align}
	The first step comes from $\lambda_{\min}(J_tJ_t^{\top}) = \lambda_{\min}(G_t) \geq \lambda / 2$ (see Lemma~\ref{lem:small-move}) and the last step comes from $g_t$ is an $\eps_0$ approximate solution to Eq.~\eqref{eq:reg-algo}. 
	The fourth step follows from Eq.~\eqref{eq:second-order2} and the fact that $\|(J_tJ_t^{\top})^{-1}\| \leq 2/\lambda$. The last step follows from $g_t$ is an $\eps_0$ ($\eps_0 \leq \sqrt{\lambda/n}$) approximate solution to the regression~\eqref{eq:reg-algo}.
	
	 Consequently, we have
	\begin{align}\label{eq:second-order8}
	\|(J_t-J_{t, t+1})J_t^{\top} (g_t - g^{\star})\|_2  
	\leq & ~ \|J_t - J_{t, t+1}\| \cdot \|J_t^{\top}\| \cdot \| g_t - g^{\star} \|_2 \notag \\
	\leq & ~ \tilde{O} (  R^{1/2} n^{1/2} / m^{1/4} ) \cdot \sqrt{n} \cdot \frac{2}{\sqrt{n\lambda}} \cdot \| f_t - y \|_2 \notag \\
	= & ~ \tilde{O} ( R^{1/2} \lambda^{-1/2} n^{1/2} /  m^{1/4} ) \cdot \|f_t - y\|_2 \notag \\
	\leq & ~ \frac{1}{6} \|f_t - y\|_2 
	\end{align}
	The second step follows from Eq.~\eqref{eq:second-order6} and \eqref{eq:second-order2} and the fact that $\|J_{t}\| \leq O(\sqrt{n})$ (see Lemma~\ref{lem:small-move})
	The last step follows from the $m \geq \Omega(n^4\lambda^{-4})$.  
	Combining Eq.~\eqref{eq:second-order3}, \eqref{eq:second-order5}, \eqref{eq:second-order11}, and \eqref{eq:second-order8}, we have proved the second claim, i.e.,
	\begin{align}
	\label{eq:second-order12}
	\|f_{t + 1} - y\|_2 \leq \frac{1}{2}\|f_t - y\|_2.
	\end{align}
	 It remains to show that $W_t$ does not move far away from $W_0$. First, we have
	 	 \begin{align}\label{eq:second-order7}
	 	 \|g_t\|_2 
	 	 \leq & ~ \|g^{\star}\|_2 + \|g_t - g^{\star}\|_2 \notag \\
	 	 \leq & ~ \|(J_tJ_t^{\top})^{-1}(f_t - y)\|_2 + \|g_t - g^{\star}\|_2 \notag \\
	 	 \leq & ~ \|(J_tJ_t^{\top})^{-1}\| \cdot \|(f_t - y)\|_2 + \|g_t - g^{\star}\|_2 \notag \\
	 	 \leq & ~ \frac{2}{\lambda}\cdot \|f_t - y\|_2 + \frac{2}{\sqrt{n\lambda}} \cdot \|f_t - y\|_2 \notag \\
	 	 \lesssim & ~ \frac{1}{\lambda} \cdot \|f_t - y\|_2
	 	 \end{align}
	 where the third step follows from Eq.~\eqref{eq:second-order2} and the last step follows from the obvious fact that $1/\sqrt{n\lambda} \leq 1/\lambda$.
		 
	 Hence, for any $r\in [m]$ and $0 \leq k\leq t$, if we use $g_{k, i}$ to denote the $i^{th}$ indice of $g_k$, then we have
	 \begin{align*}
		 \|w_r(k+1) - w_r(k)\|_2 = &~ \left\|\sum_{i=1}^{n}\frac{1}{\sqrt{m}}a_rx_r^{\top}{\bf 1}_{\langle w_r(t), x_r\rangle\geq 0}g_{k, i}\right\|_2\\
		 \leq &~ \frac{1}{\sqrt{m}}\sum_{i=1}^{n}|g_{k, i}|\\
		 \leq &~ \frac{\sqrt{n}}{\sqrt{m}}\|g_k\|_2\\
		 \lesssim &~ \frac{\sqrt{n}}{\sqrt{m}} \cdot \frac{1}{2^{k}\lambda}\|f_0 - y\|_2\\
		 \lesssim &~ \frac{n}{\sqrt{m}\lambda} \cdot \frac{1}{2^k}\\
	 \end{align*}
	 The first step follows from the updating rule, the second step follows from triangle inequalities and the fact that $a_r = \pm 1$, $\|x_r\|_2 = 1$. The third step comes from Cauchy-Schwartz inequality, and the fouth step comes from Eq.~\eqref{eq:second-order12} and Eq.~\eqref{eq:second-order7}. The last inequality comes from the fact that $\|f_0 - y\|_2 \leq O(\sqrt{n})$ (see Lemma~\ref{lem:initialization}). Consequently, we have
	 \begin{align*}
	 \|w_r(t+1) - w_r(0)\|_2 \leq &~ \sum_{k=0}^{t}\|w_r(k+1) - w_r(k)\|_2 \lesssim \sum_{k=0}^{t} \frac{n}{\sqrt{m}\lambda} \cdot \frac{1}{2^k} \lesssim \frac{R}{\sqrt{m}}.
	 \end{align*}
	 Thus we also finish the proof of the first claim.

	 It remains to give an analysis on the running time of our algorithm. In each iteration, besides evaluating function value and doing backpropagation, which generally takes $O(mnd)$ time, we also need to solve the regression problem in~\eqref{eq:reg-algo}, which takes $\tilde{O}(mnd\log(\kappa(J_tJ_t^{\top}) / \eps_0) + n^3)$ time by Lemma~\ref{lem:fast-regression}. 
	 From Lemma~\ref{lem:small-move}, we know $\|J_tJ_t^{\top}\|= \|G_t\| \leq O(n)$ and $\lambda_{\min}(J_tJ_t^{\top}) = \lambda_{\min}(G_t) \geq O(\lambda)$. Moreover, we only need to set $\eps_0 = \min\{\sqrt{\lambda/n}, 1/6 \}$. Thus the total computation cost in each iteration is $\tilde{O}(mnd + n^3)$, and the total running time to reduce the trainning loss below $\eps$ is $\tilde{O}((mnd + n^3)\log(1/\eps))$.
\end{proof}


%% file: optimization.tex
\section{Application: Convex Optimization}
\label{sec:opt}

We apply our technique to convex optimization problem. We follow the problem formulation in~\cite{pw17} and consider the problem
\begin{align*}
\min_{x} f(x)
\end{align*}
where $f$ is $\gamma$-strongly convex, $\beta$-smooth and its Hessian matrix $\nabla^2 f(x)$ is $L$ Lipschitz continuous,
\begin{definition}[$\gamma$-strongly convex]
	The function $f$ is $\gamma$-strongly convex if 
	\begin{align*}
	f(y) \geq f(x) + \langle \nabla f(x), y - x\rangle + \frac{\gamma}{2}\|y - x\|_2^2.
	\end{align*}
\end{definition}

\begin{definition}[$\beta$-smooth]
	The function $f$ is $\beta$-strongly convex if 
	\begin{align*}
	f(y) \leq f(x) + \langle \nabla f(x), y - x\rangle + \frac{\beta}{2}\|y - x\|_2^2.
	\end{align*}
\end{definition}
\begin{definition}[$L$ Lipschitz continuous Hessian]
	The Hessian matrix of function $f$ is $L$ Lipschitz continuous is
	\begin{align*}
	\| \nabla^2 f(x) - \nabla^2 f(y) \| \leq L \| x - y \|_2.
	\end{align*}
\end{definition}

As in~\cite{pw17}, we further assume we have access to 
\begin{align*}
\text{~the~square~root~of~Hessian} := \nabla^{2} f(x)^{\frac{1}{2}} ~~~ \in \R^{m \times n}
\end{align*}
with $m \geq n \poly(\log n)$. 

There are many natural and interesting examples that are valid for this assumption.
For instance, suppose the objective function has the form of $f(x) = g(Ax)$ where $A \in \R^{n \times d}$ and the function $g: \R^{n} \rightarrow \R$ has the separable form $g(Ax) = \sum_{i=1}^{n} g_i(\langle a_i, x\rangle)$, then the square root of Hessian is given by 
\begin{align*}
\nabla^{2}f(x)^{ \frac{1}{2} } = D_g(x) A \in \R^{n \times d},
\end{align*}
where $D_{g}(x) \in \R^{n \times n}$ is a diagonal matrix such that the $i,i$-th of $D_g(x)$ is $ \sqrt{ g_i''(\langle a_i, x\rangle) }$.

 For more examples, we refer interested reader to Section 3.3 in~\cite{pw17}

Naive implementation of Newton method needs to compute 
\begin{align*}
\nabla^2 f(x) = (\nabla^2 f(x)^{\frac{1}{2}})^{\top} \cdot ( \nabla^2 f(x)^{\frac{1}{2}} )
\end{align*}
and it costs $O(nd^2)$ time.
The original analysis of \textsc{NewtonSketch} in~\cite{pw17} takes $\tilde{O}(nd + d^3)$, but it requires $n \geq d\kappa^2$, where $\kappa$ is the condition number defined as $\kappa = \beta/\gamma$. 
There are many follow up work~\cite{zyr+16, ylz17, bbn19} 
intending to get rid of the extra dependence on the condition number $\kappa$.
We present an alternative approach and improve the running time to $\tilde{O}((n\log(\kappa) + d^2)d\log(1/\eps))$ by incorporating the ``fast regression solver'' introduced in this paper.

\begin{algorithm}
	\caption{Fast Newton Update} 
	\label{alg:optimization}
	\begin{algorithmic}[1]
	\Procedure{FastNewtonUpdate}{$f, x_0$} \Comment{Theorem~\ref{thm:optimization}}
		\State \Comment $x_0$ is an initial point that is satisfying $\|x_0 - x^{\star}\|_2 \leq O( \gamma / L )$
		\State $t \leftarrow 1$
		\While{$t < T$}
		\State Compute $\nabla^2 f(x_t)^{\frac{1}{2}}\in \R^{m \times n}$ and $\nabla f(x)\in \R^{n}$.
		\State Find an $1/(4\kappa)$ approximate solution $g_t \in \R^n$ to the regression problem
		\begin{align}
		\label{eq:regresion-opt}
		\min_{g_t \in \R^n} \|(\nabla^2 f(x_t)^{\frac{1}{2}})^{\top} \cdot ( \nabla^2 f(x_t)^{\frac{1}{2}} ) \cdot g_t - \nabla f(x_t)\|_2
		\end{align}
		\State $x_{t + 1} \leftarrow x_t - g_t$
		\State $t \leftarrow t + 1$
		\EndWhile
		\State \Return $x_T$
	\EndProcedure
	\end{algorithmic}
\end{algorithm}

Our algorithm is shown in Algorithm~\ref{alg:optimization}. Formally, we have
\begin{theorem}\label{thm:optimization}
	Suppose function $f$ is $\gamma$-strongly convex, $\beta$-smooth and its Hessian is $L$ Lipschitz continuous. Given an initialization point $x_0$ satisfying $\|x_0 - x^{\star}\|_2 \leq  \gamma / ( 2 L )$, there is an algorithm (procedure \textsc{FastNewtonUpdate} in Algorithm~\ref{alg:optimization}) achieves
	\begin{align}
	\label{eq:opt4}
		\|x_{t+1} - x^{\star}\|_2 \leq \frac{1}{4}\|x_t - x^{\star}\|_2 + \frac{L}{\gamma}\|x_t - x^{\star}\|_2^2,
	\end{align} 
	Consequently, in order to find an $\eps$ approxmate optimal solution, the running time is 
	\begin{align*}
	\tilde{O}\left((nd\log(\kappa) + d^3)\log(1/\eps)\right).
	\end{align*}
	Using fast matrix multiplication, the running time can be further reduced to $\tilde{O}\left((nd\log(\kappa) + d^\omega)\log(1/\eps)\right)$.

\end{theorem}

\begin{proof}
	We first analyze the correctness, and then give an analysis on the running time. Denote 
	\begin{align}
	\label{eq:opt5}
	\tilde{g}_t = \nabla^2 f(x_t)^{-1}\nabla f(x_t) = \Big( ( \nabla^2 f(x_t)^{\frac{1}{2}})^{\top} \cdot (\nabla^2 f(x_t)^{\frac{1}{2}} ) \Big)^{-1}\nabla f(x_t).
	\end{align}
	We have
	\begin{align}
	& ~ \|\nabla^2 f(x_t) (x_{t + 1} - x^{\star}) \|_2 \notag \\
	= & ~ \|\nabla^2 f(x_t) (x_{t} - x^{\star} + x_{t + 1} - x_t)\|_2\notag \\ 
	= & ~ \|\nabla^2 f(x_t) (x_t - x^{\star}) - \nabla^2 f(x_t) g_t\|_2\notag\\
	= & ~ \|\nabla^2 f(x_t) (x_t - x^{\star})  - \nabla^2 f(x_t) \tilde{g}_t + \nabla^2 f(x_t) \tilde{g}_t -\nabla^2f(x_t)  g_t\|_2 \notag\\
	\leq & ~  \|\nabla^2 f(x_t) (x_t - x^{\star})  - \nabla^2 f(x_t) \tilde{g}_t\|_2 + \|\nabla^2 f(x_t) \tilde{g}_t -\nabla^2f(x_t)  g_t\|_2 \label{eq:opt1}
	\end{align}
	For the first term
	\begin{align}
	& ~ \|\nabla^2 f(x_t) (x_t - x^{\star})  - \nabla^2 f(x_t) \tilde{g}_t\|_2 \notag \\
	= & ~ \|\nabla^2 f(x_t) (x_t - x^{\star})  - \nabla f(x_t)\|_2 \notag \\
	= & ~ \Big\| \nabla^2 f(x_t) (x_t - x^{\star})  - \int_{s=0}^{1}\nabla^2 f(x^{\star} + s(x_t - x^{\star}))(x_t - x^{\star}) \d s \Big\|_2 \notag \\
	= & ~ \Big\| \int_{s=0}^{1}\left(\nabla^2 f(x_t) - \nabla^2 f(x^{\star} + s(x_t - x^{\star}))\right) (x_t - x^{\star})\d s \Big\|_2 \notag \\
	\leq & ~ \int_{s=0}^{1}\|\nabla^2 f(x_t) - \nabla^2 f(x^{\star} + s(x_t - x^{\star}))\| \d s \cdot  \|x_t - x^{\star}\|_2\notag\\
	\leq & ~ \int_{s=0}^1 L(1 - s) \|x_t - x^{\star}\|_2\d s \cdot \|x_t - x^{\star}\|_2\notag\\
	\leq & ~ L\|x_t - x^{\star}\|_2^2. \label{eq:opt2}
	\end{align}
	The first step follows from the definition of $\tilde{g}_t$ in Eq.~\eqref{eq:opt5}, the second step follows from $\nabla f(x^{\star}) = 0$. If the Hessian is $L$ Lipschitz continuous, we have
	For the second term 
	\begin{align}
	\|\nabla^2f(x_t)  g_t - \nabla^2 f(x_t) \tilde{g}_t\|_2 
	= &~\|\nabla^2 f(x_t) g_t -\nabla f(x_t)\|_2 \notag \\
	\leq &~ \frac{1}{4\kappa}\|\nabla f(x_t)\|_2\notag\\ 
	= &~ \frac{1}{4\kappa}\cdot  \beta\|x_t - x^{\star} \|_2\notag\\
	= &~ \frac{\gamma}{4}\|   x_t - x^{\star}  \|_2 .  \label{eq:opt3} 
	\end{align}
	The first step follows from Eq.~\eqref{eq:opt5}, the second step holds since $g_t$ is an $ 1 / ( 4 \kappa ) $ approximate solution to Eq.~\eqref{eq:regresion-opt}. The third step follows from the smoothness of $f$. Consequently, we have 
	\begin{align*}
	\|x_{t+1} - x^{\star}\| 
	\leq & ~ \frac{1}{\gamma} \|\nabla^2 f(x_t)(x_{t+1} - x_t)\|_2 \\
	\leq & ~ \frac{1}{\gamma}(L\|x_t - x^{\star}\|_2^2 + \frac{\gamma}{4}\|x_t - x^{\star}\|_2)\\
	\leq & ~ \frac{1}{4}\|x_t - x^{\star}\|_2 + \frac{L}{\gamma}\|x_t - x^{\star}\|_2^2.
	\end{align*}
	The first step follows from the convexity of $f$. The second step follows from Eq.~\eqref{eq:opt1}, \eqref{eq:opt2}, and \eqref{eq:opt3}.
	Thus we prove the correctness of Eq.~\eqref{eq:opt4}.  Since we know $\kappa (\nabla^2 f(x_t)^{\frac{1}{2}}) = \sqrt{\kappa}$, the running time per iteration is 
	$\tilde{O}(nd \log (\kappa) + d^3)$ by Lemma~\ref{lem:fast-regression}. Thus we conclude the proof. 
\end{proof}
